\documentclass[preprint,12pt]{elsarticle}

\usepackage{graphicx}
\usepackage{amssymb}

\usepackage{lineno}

\usepackage{subfig}
\usepackage{algorithm}
\usepackage[noend]{algorithmic}
\usepackage{epsfig}
\usepackage{url}
\usepackage{fancyhdr}
\usepackage{tabularx}
\usepackage{amsmath}
\usepackage{graphics}
\usepackage{float}
\usepackage{lscape}
\usepackage{rotating}
\usepackage{multicol}
\usepackage{pifont}
\usepackage{caption}
\usepackage{ntheorem}
\newtheorem{theorem}{Theorem}
\newtheorem{lemma}{Lemma}
\newtheorem{hyp}{Hypothesis}
\newproof{proof}{Proof}

\journal{COMPUTERS \& OPERATIONS RESEARCH}

\begin{document}

\begin{frontmatter}


\title{Improved Quick Hypervolume Algorithm}



\author{Andrzej Jaszkiewicz}

\address{Poznan University of Technology, Faculty of Computing, Institute of Computing Science, ul. Piotrowo 2, 60-965 Poznan, andrzej.jaszkiewicz@put.poznan.pl}

\begin{abstract}
In this paper, we present a significant improvement of the Quick Hypervolume algorithm, one of the state-of-the-art algorithms for calculating the exact hypervolume of the space dominated by a set of d-dimensional points. This value is often used as the quality indicator in the multiobjective evolutionary algorithms and other multiobjective metaheuristics and the efficiency of calculating this indicator is of crucial importance especially in the case of large sets or many dimensional objective spaces. We use a similar divide and conquer scheme as in the original Quick Hypervolume algorithm, but in our algorithm we split the problem into smaller sub-problems in a different way. Through both theoretical analysis and a computational study we show that our approach improves the computational complexity of the algorithm and practical running times.
\end{abstract}

\begin{keyword}
Multiobjective optimization \sep Hypervolume indicator \sep Computational complexity analysis
\end{keyword}

\end{frontmatter}


\section{Introduction}

In this paper, we consider the problem of calculating the exact hypervolume of the space dominated by a set of d-dimensional points. This hypervolume is often used as the quality indicator in the multiobjective evolutionary algorithms (MOEAs) and other multiobjective metaheuristics (MOMHs), where the set of points corresponds to images in the objective space of the solutions generated by a MOMH. Multiple quality indicators have been proposed in the literature, however, the hypervolume indicator has the advantage of being compatible with the comparison of approximation sets based on the dominance relation (see \cite{Zitzler2003} for details) and is one of the most often used indicators. The hypervolume indicator may be used a posteriori to evaluate the final set of solutions generated by a MOMH e.g. for the purpose of a computational experiment comparing different algorithms or to tune parameters of a MOMH. Some authors proposed also indicator-based MOMHs that use the hypervolume to guide the work of the algorithms \cite{Zitzler2004,Jiang2015}.  

The exact calculation of the hypervolume may become, however, computationally demanding especially in the case of large sets in many dimensional objective spaces. Thus the exact calculation of the hypervolume obtained a significant interest from the research community \cite{Russo2014,Russo2016,Beume2009,ISI:While2012,Lacour2017347,Chan2013}. According to the recent study of Lacour et al. \cite{Lacour2017347} the state-of-the-art algorithms for the exact calculation of the hypervolume are Quick Hypervolume (QHV) \cite{Russo2014,Russo2016}, Hypervolume Box Decomposition Algorithm (HBDA) \cite{Lacour2017347} and Walking Fish Group algorithm (WFG) \cite{While2102}. From the theoretical point of view the currently most efficient algorithm in the case $d \geq 4$ in terms of the worst case complexity ($\mathcal{O}(n^\frac{d}{3})\text{polygon}(n)$) is by Chan \cite{Chan2013}. To our knowledge, there is currently, however, no available implementation of this approach and no evidence of its practical efficiency \cite{Lacour2017347}.

In this paper, we improve QHV algorithm by modifying the way the problem is split into smaller sub-problems. This modification although may seem relatively simple significantly improves the computational complexity of the algorithm and practical running times. Since our work is based on relatively recently published results we do not give in this paper an extended overview of the applications of the hypervolume indicator and the algorithms for the hypervolume calculation. Instead we refer an interested reader to \cite{Russo2014,Russo2016,Beume2009, ISI:While2012,Lacour2017347,While2102} for recent overviews of this area.

The paper is organized in the following way. In the next section, we define the problem of the hypervolume calculation. In section 3, the improved Quick Hypervolume (QHV-II) algorithm is proposed. The computational complexity of QHV-II algorithm is analyzed and compared to QHV in section 4. In section 5, a computational study is presented. The paper finishes with conclusions and directions for further research.

\section{Problem formulation}
Consider a d-dimensional space $\mathbb{R}^d$ that will be interpreted as the space of $d$ maximized objectives.  

We say that a point $s^1 \in \mathbb{R}^d$ \emph{dominates} a point $s^2 \in \mathbb{R}^d$ if, and only if, $s^1_j \geq s^2_j\  \forall \, j \in \{1,\ldots,d\} \wedge \exists \, j \in \{1,\ldots,d\}: s^1_j > s^2_j$. We denote this relation by $s^1 \succ s^2$.

We will consider \emph{hypercuboids} in $\mathbb{R}^d$ parallel to the axes, defined by two extreme points $r_* \in \mathbb{R}^d$ and $r^* \in \mathbb{R}^d$ such that $H(r^*, r_*) = \{s \in \mathbb{R}^d\ \mid \forall \, j \in \{1,\ldots,d\} \, \, {r_*}_j \leq s_j \leq r^*_j \}$.

Consider a finite set of points $S \subset H(r^*, r_*)$. The hypervolume of the space dominated by $S$ within hypercuboid $H(r^*, r_*)$, denoted by $\mathcal{H}(S, H(r^*, r_*))$ is the Lebesgue measure of the set $\bigcup\limits_{s \in S} H(s, r_*)$. The introduction of $r^*$ may seem superfluous since it does not influence the hypervolume, however, such definition will facilitate further explanation of the algorithms which are based on the idea of splitting the original problem into sub-problems corresponding to smaller hypercuboids.

\section{Quick Hypervolume II algorithm}
In this section we propose a modification of the QHV algorithm proposed by Russo and Francisco \cite{Russo2014,Russo2016}. We call this modified algorithm QHV-II. Both QHV and QHV-II are based on the following observations:
\begin{enumerate}
\item $\forall \, s' \in S \quad \mathcal{H}(S, H(r^*, r_*)) = \mathcal{H}(s', r_*) + \mathcal{H}\Big(\big(\bigcup\limits_{s \in S \setminus \{s'\}} H(s, r_*)\big) \setminus H(s', r_*)\Big)$, i.e. the hypervolume of the space dominated by $S$ is equal to the hypervolume of the hypercuboid defined by a single point $s' \in S$ and $r_*$, i.e. $\mathcal{H}(s', r_*)$, plus the hypervolume of the area dominated by the remaining points, i.e. $S \setminus \{s'\}$ excluding the area of hypercuboid $H(s', r_*)$.
\item The region $H(r^*, r_*) \setminus H(s', r_*)$ may be defined as a union of non-overlapping hypercuboids $\{H_1,...,H_L\}$.
\item Consider a point $s^1 \notin H(r^*, r_*) \, \wedge \, s^1 \succ r_* $. The hypervolume of the space dominated by $s^1$ within $H(r^*, r_*)$ is equal to the hypervolume of the space dominated by the projection of $s^1$ onto $H(r^*, r_*)$. The projection means that the coordinates of the projected point larger than the corresponding coordinates of $r^*$ are replaced by the corresponding coordinates of $r^*$. 
\end{enumerate}

The above observations immediately suggest the possibility of calculating the hypervolume in a recursive manner with Algorithm~\ref{algoGQHV}. The algorithm selects a pivot point, calculates the hypervolume of the area dominated by the pivot point, and then splits the problem of calculating the remaining hypervolume into a number of sub-problems. If the number of points is sufficiently small it uses simple geometric properties to calculate the hypervolume.

\begin{algorithm}[!ht]
\caption{\texttt{General QHV}}\label{algoGQHV}
\begin{algorithmic}
\STATE Parameters $\downarrow$: $H(r^*, r_*), S \subset H(r^*, r_*)$ 
\vspace*{1\baselineskip}
\IF {$S$ contains one or two points}
\STATE Calculate $\mathcal{H}(S, H(r^*, r_*))$ using simple geometric properties
\STATE $HyperVolume \leftarrow \mathcal{H}(S, H(r^*, r_*))$ 
\ELSE
\STATE Select a pivot point $s' \in S$
\STATE $HyperVolume \leftarrow \mathcal{H}(s', r_*)$ 
\STATE Split $H(r^*, r_*) \setminus H(s', r_*)$ into a set of non-overlapping hypercuboids $\{H_1,\dots,H_L\}$.
\FORALL {$H_l \in \{H_1,\dots,H_L$\}}
\STATE Construct set $S_l$ containing the points dominating $r^l_*$ and if necessary projected onto $H_l$
\STATE $HyperVolume \leftarrow HyperVolume + \texttt{QHV}(H_l, S_l)$
\ENDFOR
\ENDIF
\RETURN $HyperVolume$
\end{algorithmic}
\end{algorithm}

Russo and Francisco \cite{Russo2014,Russo2016} propose to split the region $H(r^*, r_*) \setminus H(s', r_*)$ into $2^d - 2$ hypercuboids corresponding to each possible combination of the comparisons on each objective, where a coordinate may be $<$ or $\geq$ than the corresponding coordinate of the pivot point $s'$, with the exception of the two combinations corresponding to the areas dominated and dominating $s'$. Each such hypercuboid may be defined by a binary vector where $0$ at $j$-th position means that $s_j < s'_j$ and $1$ at $j$-th position means that $s_j \geq s'_j$. We will call such hypercuboids \emph{basic hypercuboids}.

We propose a different splitting scheme. We split the region $H(r^*, r_*) \setminus H(s', r_*)$ into $d$ hypercuboids defined in the following way:
\begin{itemize}
\item $H_1$ is defined by the condition $s_1 \geq s'_1$
\item \dots
\item $H_j$ is defined by the conditions $s_l < s'_l \quad \forall \, l=1,\dots,j-1 \land s_j \geq s'_j$
\item \dots
\end{itemize}
In other words, the hypercuboids are defined not by binary vectors but by the following schemata of binary vectors:
\begin{itemize}
\item $v_1 = 1*\dots*$
\item \dots
\item $v_i = 0\dots01*\dots*$, with $1$ at $j$-th position
\item \dots
\item $v_d = 0\dots01$
\end{itemize}
where $*$ means any symbol either $0$ or $1$. Each of these hypercuboids is obviously an union of a number of the basic hypercuboids. 

The difference between the splitting schemes in QHV and QHV-II is graphically illustrated in Figure \ref{fig:AlgComp} for the 3-objective case. In this case, there are 6 basic hypercuboids. The colors describe the hypercuboids corresponding to the different sub-problems. Please note, that in the case of QHV-II the hypercuboid defined by the condition $s_1 \geq s'_1$ contains also the region dominating $s'$ but this region does not contain any points. In addition, the arrows indicate the directions of projections of the points onto the hypercuboids.

\begin{figure}[htb]
\centering
\hspace*{1.5em}\raisebox{\dimexpr-.5\height-1em}
  {QHV}%
\hspace*{13.5em}\raisebox{\dimexpr-.5\height-1em}
  {QHV-II}%
\\[\medskipamount]
\hspace*{1.5em}\raisebox{\dimexpr-.5\height-1em}
  {\includegraphics[scale=0.45]{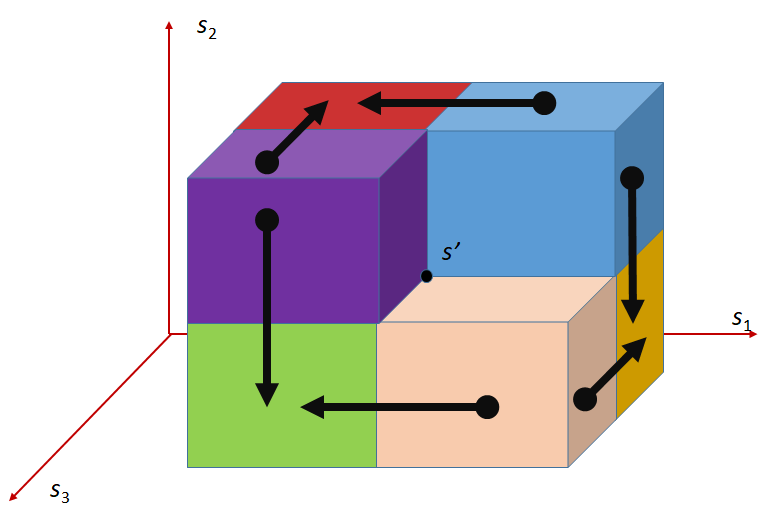}}%
\hspace*{1.5em}\raisebox{\dimexpr-.5\height-1em}
  {\includegraphics[scale=0.45]{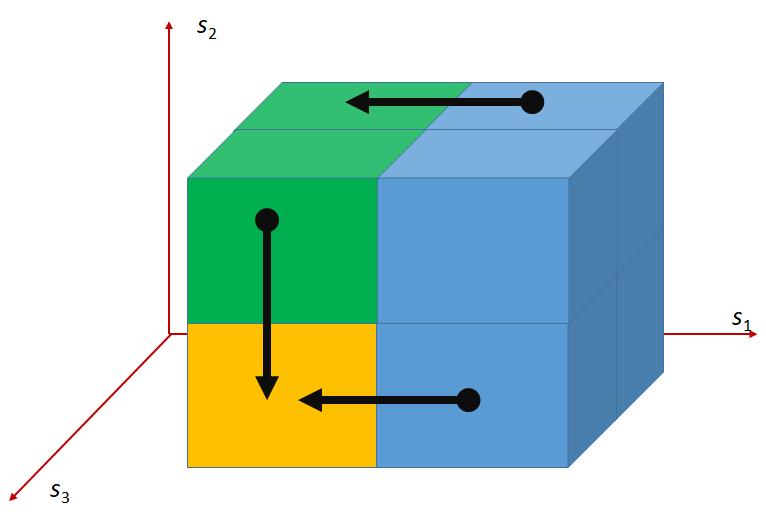}}%
\\[\medskipamount]
\hspace*{1.5em}\raisebox{\dimexpr-.5\height-1em}
  {$d=10$}%
\hspace*{15.5em}\raisebox{\dimexpr-.5\height-1em}
  {$   $}%
\medskip

\caption{The comparison of the splitting schemes in QHV and QHV-II}
\label{fig:AlgComp}
\end{figure}

Alike proposed in \cite{Russo2014,Russo2016} as the pivot point we select the point $s' \in S$ with the maximum $\mathcal{H}(s', r_*)$. 

Please note, that the points projected onto a hypercuboid $H_l$ may become dominated since some coordinates are replaced with lower values. Russo and Francisco \cite{Russo2014,Russo2016} propose to explicitly remove the dominated points e.g. with the algorithm proposed in \cite{Bentley1980}. We do not use this step in QHV-II since we did not find it practically beneficial in the preliminary computational experiments. Note, however, that alike Russo and Francisco \cite{Russo2014,Russo2016} we implemented only a naive method for the removal of the dominated points by comparing all pairs of points. This could perhaps be improved by using more advanced methods. Please note, however, that the pivot point $s'$ selected in the way described above is guaranteed to be non-dominated within $S$. Furthermore, while assigning points to sets $S_l$ each point is compared to $s'$ and the points dominated by $s'$ may be removed. It does not guarantee an immediate removal of all dominated points but finally all dominated points will be removed by the algorithm because each of the dominated points will be dominated by one of the selected pivots or eliminated while using the simple geometric properties when the number of points is sufficiently small.

Please also note, that as suggested in \cite{Russo2014,Russo2016} the projected points do not need to be constructed explicitly, but their coordinates may be calculated on demand to reduce the memory requirements.

\section{Theoretical Analysis} \label{TheorAnal}

The analysis of the computational complexity of such complex algorithms like QHV or QHV-II is very difficult. Russo and Francisco \cite{Russo2014} present a quite complex analysis of the computational complexity of QHV algorithm. However, the authors needed to make some strong assumptions like the uniform distribution of points on a hyperplane or the surface of a hypersphere to prove some main results. Furthermore their analysis in general shows that for large number of points, most of the projected points become dominated and may be removed. Thus the time needed for the removal of the dominated points has the main influence on the asymptotic behavior of the algorithm. Realistic data sets are, however, much smaller and the projected non-dominated points significantly influence running times. Indeed in our theoretical analysis we show that in the best case when all projected points are dominated and removed QHV and QHV-II have the same computational complexity. However, if all or a fraction of the projected points are preserved QHV-II starts to outperform QHV.

In this paper we use more standard tools like the recurrence solving and the Akra-Bazzi method \cite{Akra1998} to analyze the computational complexity of QHV-II and QHV is some specific cases. We present results for the worst and the best cases. Since it is very difficult to analyze the average case behavior of such complex algorithms we analyze a number of specific intermediate cases for which we can obtain the complexity of QHV and QHV-II.


\subsection{Worst case}

To analyze the worst case we follow the approach proposed in \cite{Russo2014}. The worst case is when all points except of the pivot point are assigned to each sub-problem. In this case, we can discard the time of removal of the dominated points and use the relation:
\begin{equation}
T(n) = dT(n-1)
\end{equation}
where $T(n)$ is the numbers of the comparisons needed to process a set of $n$ points. Solving the above recurrence:
\begin{equation}
T(n) = cd^{n-1}
\end{equation}
where $c$ is a constant. Thus QHV-II has the worst case time complexity $\mathcal{O}(d^{n-1})$

For the original QHV algorithm we get:
\begin{equation}
T(n) = (2^d-2)T(n-1)
\end{equation}
Solving the above recurrence:
\begin{equation}
T(n) = c(2^d-2)^{n-1}
\end{equation}
where $c$ is a constant. Thus QHV has the worst case time complexity $\mathcal{O}(2^{d(n-1)})$ which is worse than the worst case time complexity of QHV-II.


\subsection{Best case}

In the best case points are distributed uniformly in all basic hypercuboids and all projected points become dominated and are removed. So, each of the $2^d-2$ basic hypercuboids contains $n/(2^d-2)$ points.

To analyze the complexity of the algorithm we use the Akra-Bazzi method \cite{Akra1998} for the analysis of the divide and conquer algorithms of the form:
\begin{equation}
T(n)=
\begin{cases}
      T_0, & \text{if}\ n \leq n_0 \\
      \sum_{k=1}^K \, a_k T(n/b_k) + g(n), & \text{otherwise}
\end{cases}
\end{equation}
where $K$ is a constant, $a_k > 0$ and $b_k > 1$ are constants for all k, $g(n) = \Omega(n^c)$, and $g(n) = \mathcal{O}(n^d)$ for some constants $0 < c \leq d$. In this case:
\begin{equation}
T(n)=\theta \big( n^{p^o}(1 + \int_{1}^{n}\frac{g(u)}{u^{{p^o}+1}}du) \big)
\end{equation}
where $p^o$ is the unique real solution to the equation:
\begin{equation} \label{eq:p}
\sum_{k=1}^K \, \frac{a_k}{{b_k}^p} = 1
\end{equation}

\subsubsection{QHV-II}

In QHV-II the problem is split into $d$ sub-problems. The $k$-th sub-problem corresponds to the schema with $d-k$ positions that can be either $0$ or $1$. Thus there are $2^{d-k}$ binary vectors and thus $2^{d-k}$ basic hypercuboids directly (i.e. without projection) assigned to the $k$-th sub-problem, for $k=2,\dots,d$. For $k=1$ we have $2^{d-1}-1$ directly assigned basic hypercuboids because we discard the basic hypercuboid corresponding to the binary vector 11\dots1 and containing points dominating the pivot point.

In this case, $K = d$, $a_k = 1$, $b_k = (2^d-2)/(2^{d-k})$, $k=2,\dots,d$, $b_1 = (2^d-2)/(2^{d-1}-1)$. Equation \ref{eq:p} has the solution $p^o = 1$. If we discard the removal of the dominated points (i.e. assume that all of them where removed through the comparison to the pivot point) then $g(n) = n$ and: 
\begin{equation}
\begin{split}
T(n)=\theta \big( n(1 + \int_{1}^{n}\frac{u}{u^2}du) \big) = \Theta(n\log n)
\end{split}
\end{equation}

The dominated points can be removed in $\Theta(n \log^{d-2}n)$ time \cite{Russo2014,Bentley1980}. So if we take into account the time needed for the removal of the dominated points and the time needed to compare each point to the pivot point then $g(n)=n + \Theta(n \log^{d-2}n) = \Theta(n \log^{d-2}n)$ and we get
\begin{equation}
\begin{split}
T(n)=\theta \big( n(1 + \int_{1}^{n}\frac{\Theta(u\log^{d-2}u)}{u^2}du) \big) = \\ 
= \Theta(n \log^{d-1}n)
\end{split}
\end{equation}


\subsubsection{QHV}

In the same way we may analyze the best case behavior of the original QHV algorithm. In QHV the problem is split into $d-1$ classes of equivalent sub-problems. The $k$-th class contains $\binom{d}{k}$  sub-problems with binary vectors having $k$ symbols $1$. Each subproblem contains just one basic hypercuboid.

In this case, $K=d-1$, $a_k = \binom{d}{k}$, $b_k = 2^d-2$. In this case, we also get $p^o = 1$ and the same time $T(n)= \Theta(n\log n)$ or $T(n)= \Theta(n \log^{d-1}n)$, without or with the explicit removal of the dominated points, respectively. So, in the best case the time complexity of QHV and QHV-II is the same.


\subsection{Intermediate case}
Consider the case when the points are equally distributed in all basic hypercuboids but no projected points are removed. In other words each sub-problem will contain all points belonging to its hypercuboid and all projected points.

\subsubsection{QHV-II}

In this case, $a_k = 1$ and $b_k = 2$ since each sub-problem will contain half of the points better on one objective than the pivot point. Equation \ref{eq:p} has the solution $p^o = \log_2 d$. Since we can discard the removal of the dominated points $g(n) = n$ and

\begin{equation}
\begin{split}
T(n)=\theta \big( n^{\log_2 d}(1 + \int_{1}^{n}\frac{u}{u^{log_2 d + 1}}du) \big) = \\
= \Theta \big( n^{\log_2 d}(1 + \Theta(n^{1-log_2 d})) \big) = \Theta(n^{\log_2 d})
\end{split}
\end{equation}


Furthermore, since the points dominated by the pivot point may be removed without any additional comparisons QHV-II has time complexity $\mathcal{O}(n^{\log_2 d})$ with the considered distribution of points.

\subsubsection{QHV}

In the case of the original QHV algorithm $a_k = \binom{d}{k}$, $b_k = (2^d-2) / (2^{d-k} - 1)$, since there are $(2^{d-k} - 2)$ other basic hypercuboids projected onto a basic hypercuboid with $k$ symbols $1$. 

Unfortunately, we were not able to solve analytically equation \ref{eq:p} in this case. In Table \ref{tab:p} we present values of $p^o$ obtained numerically for different numbers of objectives $d$. However, we can consider an approximate model assuming that $n$ points are uniformly distributed among all basic hypercuboids including the two hypercuboids dominated and dominating the pivot point $s'$. For larger values of $d$ the differences between the exact and the approximate models become very small as illustrated in Table \ref{tab:p} and since the approximate model differs in only two among $2^d$ basic hypercuboids it asymptotically converges to the exact model with growing $d$. Please note, that although Table \ref{tab:p} suggest that $p^o$ is always higher for the approximate model we do not have a formal proof of this fact. For the approximate model $a_k = \binom{d}{k}$, $k=1,\dots,d$, $b_k = (2^d) / (2^{d-k})$ and equation \ref{eq:p} is expressed as:
\begin{equation}
\sum_{k=1}^K \, \frac{\binom{d}{k}}{\big( (2^d) / (2^{d-k}) \big) ^p} = (2^{-p} + 1)^d - 1 = 1
\end{equation}
and has solution:
\begin{equation}
p=-\log_2(2^{1/d}-1)
\end{equation}
thus:
\begin{equation}
\begin{split}
T(n)=\theta \big( n^{-\log_2(2^{1/d}-1)}(1 + \int_{1}^{n}\frac{u}{u^{-\log_2(2^{1/d}-1)+1}}du) \big) = \\
= \theta \big( n^{-\log_2(2^{1/d}-1)}(1 + \Theta(n^{2 + log_2(2^{1/d}-1)})) \big) = \\
= \Theta(n^{-\log_2(2^{1/d}-1)})
\end{split}
\end{equation}


\begin{theorem}
$\Theta(n^{\log_2 d}) < \Theta(n^{-\log_2(2^{1/d}-1)})$, i.e. the time complexity of QHV-II in the considered case is lower than the time complexity of QHV assuming the approximate model.
\end{theorem}

This theorem is the direct consequence of Lemma \ref{lemma1}.

\begin{lemma}
\label{lemma1}
$\log_2 d < -\log_2(2^{1/d}-1)$
\end{lemma}

\begin{proof}
\begin{equation}
\begin{split}
\log_2 d < -\log_2(2^{1/d}-1) \implies \\
\log_2 d + \log_2(2^{1/d}-1) < 0 \implies \\
\log_2 (d2^{1/d}-d) < 0  \implies \\
d2^{1/d}-d < 1
\end{split}
\end{equation}
Using Laurent series extension
\begin{equation}
d2^{1/d}-d = d +\log 2 + \frac{\log^2 2}{2d} + \mathcal{O}((1/d)^2) - d < 1
\end{equation}
$\square$
\end{proof}

\begin{table}[t]
\caption{Comparison of $p^o$ values for the exact and the approximate models}
\begin{center}
\label{tab:p}
\begin{tabular}{ccc}
$d$ & $p^o$ exact model & $p^o$ approximate model \\ \hline
2 & 1	&	1.2715 \\
4 & 2.2942	&	2.4019 \\
6 & 2.9920 & 3.0295 \\
8 & 3.4543 & 3.4658 \\
10 & 3.7971 & 3.8004 \\
12 & 4.0709 & 4,0718 \\
\end{tabular}
\end{center}
\end{table}

\subsection{Constant fraction of the removed points}

In practice some of the projected points are removed. The fraction of the preserved (not removed) projected points depends in general on the number of points and the number of projected objectives. We may consider, however, a simplified model where the fraction of the preserved points is a constant $0 \leq C \leq 1$. Please note, that for $C=0$ we get the best case considered above, and for $C=1$ the intermediate case considered above. In this case:
\begin{equation}
\begin{cases}
b_k = \frac{2^d-2}{2^{d-i} + C(2^{d-1} - 2^{d-i} - 1)}, k= 2,\dots,d \\
b_1 = \frac{2^d-2}{2^{d-i} - 1}
\end{cases}
\end{equation}
for QHV-II, and:
\begin{equation}
b_k = \frac{2^d-2}{1 + C (2^{d-k} - 2)}, k= 1,\dots,d
\end{equation}
for QHV. Unfortunately, we did not manage to solve equation \ref{eq:p} in these cases analytically. Numerical analysis with the use of non-linear solvers suggests that the following hypothesis is true:

\begin{hyp}
\label{h:1}
$p^o$ is lower for QHV-II than for QHV, i.e. the time complexity of QHV-II is lower than the time complexity of QHV in the considered case, for any $C > 0$.
\end{hyp}

Namely, we used a non-linear solver to solve the following optimization problem:
\begin{equation}
\begin{split}
\min (p^o_{QHV}-p^o_{QHV-II}) \\
s.t. \\
\textnormal{equation } \ref{eq:p} \textnormal{ for QHV} \\
\textnormal{equation } \ref{eq:p} \textnormal{ for QHV-II} \\
\end{split}
\end{equation}
with variables $0 \leq C \leq 1, d \geq 2, p^o_{QHV} \geq 1,p^o_{QHV-II} \geq 1$. The optimum solution was always $C = 0$,  $p^o_{QHV} = 1$, $p^o_{QHV-II} = 1$ for any starting solution tested. Furthermore, in Table \ref{tab:p numerical} we present some exemplary values of $p^o$ for various numbers of objectives and values of $C$ obtained by solving equation \ref{eq:p} numerically. Of course, this is still not a formal proof of Hypothesis \ref{h:1}. 

\begin{table}[t]
\caption{Comparison of $p^o$ values for QHV-II and QHV}
\begin{center}
\label{tab:p numerical}
\begin{tabular}{cccc}
$d$ & $C$ & $p^o_{QHV-II}$ & $p^o_{QHV}$ \\ \hline
4	&	0,9	&	1,8635	&	2,0928	\\
4	&	0,5	&	1,4279	&	1,5071	\\
4	&	0,1	&	1,0837	&	1,0957	\\
12	&	0,9	&	3,1897	&	3,6731	\\
12	&	0,5	&	2,0706	&	2,5321	\\
12	&	0,1	&	1,2548	&	1,5445	\\
20	&	0,9	&	3,8092	&	4,3628	\\
20	&	0,5	&	2,3820	&	3,0746	\\
20	&	0,1	&	1,3563	&	1,9701	\\
\end{tabular}
\end{center}
\end{table}

\subsection{Average case with no removal of the dominated points}
In QHV-II each of the $d$ sub-problems corresponds to a split based on the value of one objective in the pivot point which allows us to follow the analysis of well-known algorithms like the binary search or Quicksort. Assuming that the pivot point is selected at random with the uniform probability and that the ranks of a given point according to particular objectives are independent variables:
\begin{equation} \label{eq:a1}
T(n)= n + d \frac{1}{n} \sum_{k=1}^{n-1}T(k) \\
\end{equation}
Multiplying both sides by $n$:
\begin{equation} \label{eq:a2}
nT(n)= n^2 + d \sum_{k=1}^{n-1}T(k) \\
\end{equation}
Assuming that $n \geq 2$:
\begin{equation} \label{eq:a3}
(n-1)T(n-1)= (n-1)^2 + d \sum_{k=1}^{n-2}T(k) \\
\end{equation}
Subtracting equations \ref{eq:a2} and \ref{eq:a3}:
\begin{equation} \label{eq:a4}
\begin{split}
nT(n) - (n-1)T(n-1) = \\
= n^2 - (n-1)^2 + d\sum_{k=1}^{n-1}T(k) - d\sum_{k=1}^{n-2}T(k) 
\end{split}
\end{equation}
Simplifying:
\begin{equation} \label{eq:a5}
T(n) = \frac{2n-1}{n} + \frac{n+d-1}{n}T(n-1)
\end{equation}
The above recurrence has the solutions: 
\begin{equation} \label{eq:c3}
T(n) = c_1\frac{\Gamma(d+n)}{\Gamma(n+1)} - c_2n - c_3
\end{equation}
where:
\begin{equation} \label{eq:c3a}
\begin{split}
c_1 = c\frac{(d-2)(d-1) + d^2}{(d-2)(d-1)\Gamma(d+1)} \\
c_2 = \frac{2}{d-2} \\
c_3 = \frac{d}{(d-1)(d-2)}
\end{split}
\end{equation}
(with $c$ being an arbitrary constant), which can be shown by substituting $T(n)$ and $T(n-1)$ with \ref{eq:c3} in \ref{eq:a5}

Since
\begin{equation} \label{eq:c4}
\lim_{n\to \infty} \frac{\Gamma(n+\alpha)}{\Gamma(n)n^\alpha} = 1
\end{equation}
thus
\begin{equation} \label{eq:c4}
\lim_{n\to \infty} \frac{\Gamma(n+\alpha)}{\Gamma(n)} = n^\alpha
\end{equation}
\begin{equation} \label{eq:c4}
T(n)= \Theta(n^{d-1})
\end{equation}
This analysis was based on the assumption that the pivot point is selected at random while in QHV and QHV-II we select a point with maximum $\mathcal{H}(s, r_*)$. This increases the chance of selecting points with a good balance of the values of all objectives with ranks closer to $n/2$ on each objective which may improve the practical behavior of the algorithms. 

Similar analysis of QHV is much more difficult, we can prove however that under the same assumptions the following relation holds:
\begin{theorem}
$T'(n) > T(n)$, where $T'(n)$ and $T(n)$ are the numbers of comparisons in QHV and QHV-II algorithms, respectively.
\end{theorem}

\begin{proof}
In QHV each node is split into $2^d-2$ sub-problems. Among them are the $d$ sub-problems corresponding to the $d$ basic hypercuboids with only one objective better or equal to the corresponding value in the pivot point. Consider one of such basic hypercuboids better or equal to the corresponding value in the pivot point only on objective $j$. All points with equal or better values on objective $j$ either belong to or are projected on this hypercuboid. Thus, each these $d$ sub-problems is equivalent to one of the $d$ sub-problems used in QHV-II. In result:
\begin{equation}
\begin{split}
T'(n)= n + d \frac{1}{n} \sum_{k=1}^{n-1}T(k) + T''(n) = \\
= T(n) + T''(n) > T(n)
\end{split}
\end{equation}
where $T''(n)$ is the number of comparisons needed to process the remaining sub-problems other than the $d$ sub-problems discussed above.
$\square$
\end{proof}

\subsection{Discussion}
We have shown that the worst case time complexity of QHV-II is better than of QHV. On the other hand, the best case time complexity of the two algorithms is equal. We have analyzed also a number of intermediate cases each time showing that the time complexity of QHV-II is better than of QHV which supports the hypothesis that in practical cases QHV-II performs better than QHV.




\section{Computational experiment}

In order to test the proposed method we used the test instances proposed in \cite{Lacour2017347}. The set of instances includes linear, concave, convex, and hard instances with up to 10 objectives and up to 1000 points. In addition we used also instances with points uniformly distributed over a multi-dimensional spherical surface generated by us in a way proposed by Russo and Francisco \cite{Russo2014}. Among the hard instances there is just one instance of each size. In other cases, there are 10 instances of each size. The presented results are the average values over 10 instances of each type and size in all cases with the exception of the hard instances. All test instances used in this experiment, as well as the source code and the detailed numerical results, are available at https://sites.google.com/view/qhv-ii/qhv-ii.

The main goal of the computational experiment is to show that QHV-II performs less operations than QHV. Please note, that our goal is not to present a very efficient code for calculation of the hypervolume. We have developed only a simple implementation of QHV-II in C++. In the same programming language we have implemented also QHV algorithm. Both implementations share majority of the code. We compare it also to the implementation of QHV available at http://web.tecnico.ulisboa.pt/luis.russo/QHV/ developed by the authors of the algorithm. We call this implementation oQHV (the original implementation of QHV). oQHV implementation is much faster than our implementation of QHV (and thus of QHV-II) because of several reasons:
\begin{itemize}
\item oQHV is implemented in C and the code is highly optimized for efficiency. Our implementation is made in C++ and is relatively simple. In particular, we do not use any low level code optimizations used in oQHV, and we use object-oriented techniques, e.g. classes from the STL library, which are very convenient, but can make the code less efficient.
\item oQHV uses additional methods, namely HSO\cite{While2006} and IEX \cite{Wu2001IEX} for the sub-problems with small numbers of points, because these methods are faster in such cases. We do not use these additional methods in our implementation.
\end{itemize}


To make sure that our understanding of QHV is correct we modified the code of oQHV. We call this modified implementation oQGHV-s (oQHV simplified). In oQGHV-s we do not use the additional methods (HSO and IEX) for the sub-problems with small numbers of points. In other words, oQGHV-s implements the same algorithm as our implementation of QHV. We have compared the numbers of visited nodes in the recursive tree in both our implementation of QHV and oQHV-s. The numbers of visited nodes were the same for smaller instances and differed slightly (only in once case above 0.1\%) for larger instances. These small differences where probably caused by small numerical perturbations since in both implementations the objective values are stored as the floating point numbers. On the other hand the running times of our implementation of QHV were on average 3 times higher due to the differences in the programming languages and code optimization. Of course, we have also checked that our implementations of QHV-II and QHV returns the same values of the hypervolume that the other tested algorithms.

The numbers of visited nodes in the case of QHV-II and oQHV are not directly comparable since oQHV cuts the recursive tree at some levels and applies the additional methods (HSO or IEX) for the small sub-problems. To make a fair comparison in Figures \ref{fig:opsL} to \ref{fig:opsUS} we compare the original implementation of QHV and our implementation of QHV-II with the recursive trees cut in both cases at the level at which the original implementation of QHV would use the additional methods. In other words, the sub-problems for which the original implementation of QHV would solve with the additional methods were completely skipped. We call these methods oQHV-cut and QHV-II-cut. We report the numbers of visited (internal) nodes and the number of leafs, i.e. the number of calls to the additional methods (HSO or IEX). The number of visited nodes was on average 2,21 (33,5 maximum) times higher in oQHV-cut and the number of leafs was on average 5,39 (46 maximum) times higher. These numbers were almost always lower for QHV-II-cut expect of 4 hard instances with 6 objectives were the numbers of visited nodes were slightly higher for QHV-II-cut. The relative differences were generally higher for instances with higher number of objectives. These results confirm that QHV-II constructs smaller recursive trees and allows to predict that a more advanced implementation of QHV-II, i.e. with a more optimized code and with the use of the additional methods for small sub-problems, could be several times faster than the original implementation of QHV for the tested instances. 

Please note also, that QHV-II behaves much more predictably than QHV on hard instances. The number of visited nodes and leafs in QHV-II grows gradually with the growing number of points, while in QHV it is sometimes higher for smaller sets.

Although our goal is not to develop a very efficient implementation of QHV-II our simple code is already competitive to the state-of-the-art implementations in some cases. In Figures \ref{fig:CPUL} to \ref{fig:CPUUS} we present running times of our implementation of QHV-II, the original implementation of QHV (oQHV), as well as of the non-incremental version of HBDA (HBDA-NI) \cite{Lacour2017347} and WFG \cite{While2102} algorithms. oQHV and QHV-II were run by ourselves on Intel Core i7-5500U CPU at 2.4 GHz. For HBDA-NI and WFG, in the case of the linear, concave, convex, and hard instances, we used results obtained in \cite{Lacour2017347}. Since a slower CPU was used in that experiment we re-scaled the running times by a factor of 2,5 obtained by comparing the running times in a number of test runs. For uniform spherical instances, we run ourselves WFG algorithm using the code available at http://www.wfg.csse.uwa.edu.au/hypervolume/. Unfortunately, because of some technical problems, we were not able run the code of HBDA-NI, thus for these instances this method was not used. These results confirm that even our simple implementation of QHV-II is already competitive to the state-of-the-art codes in some cases, especially for hard instances with 10 objectives. Let us remind, that the results presented above suggest that the implementation of QHV-II could be further significantly improved by adding a low level code optimization, and using the additional methods for the small sub-problems.

\begin{figure}[htb]
\centering
\hspace*{1.5em}\raisebox{\dimexpr-.5\height-1em}
  {$d=6$}%
\hspace*{13.5em}\raisebox{\dimexpr-.5\height-1em}
  {$d=8$}%
\\[\medskipamount]
\hspace*{0.5em}\raisebox{\dimexpr-.5\height-1em}
  {\includegraphics[scale=0.58]{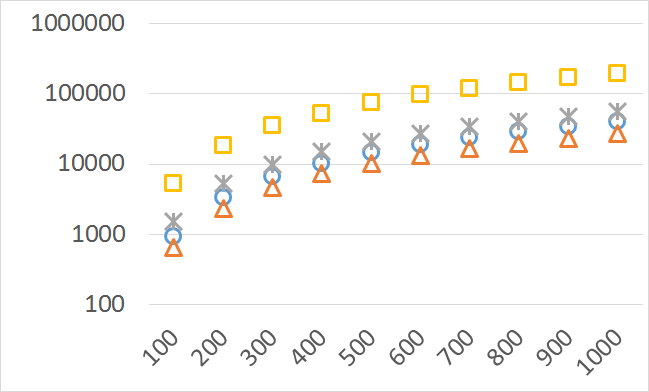}}%
\hspace*{0.2em}\raisebox{\dimexpr-.5\height-1em}
  {\includegraphics[scale=0.58]{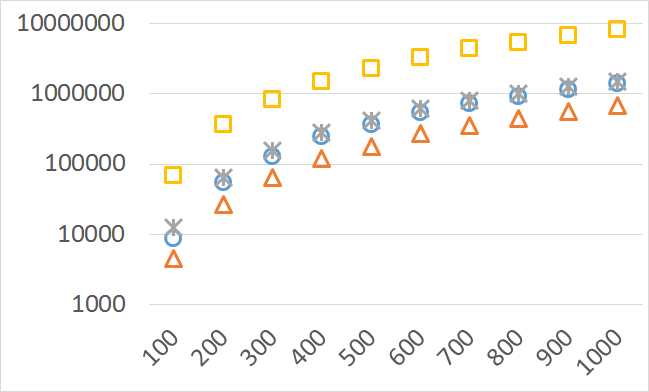}}%
\\[\medskipamount]
\hspace*{1.5em}\raisebox{\dimexpr-.5\height-1em}
  {$d=10$}%
\hspace*{15.5em}\raisebox{\dimexpr-.5\height-1em}
  {$   $}%
\\[\medskipamount]
\hspace*{0.5em}\raisebox{\dimexpr-.5\height-1em}
  {\includegraphics[scale=0.58]{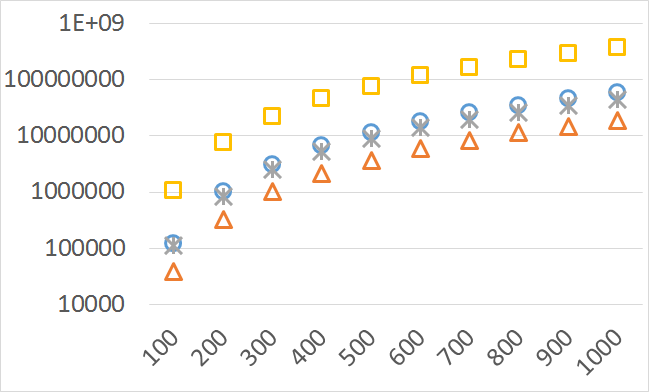}}%
\hspace*{0.2em}\raisebox{\dimexpr-.5\height-1em}
  {\includegraphics[scale=0.58]{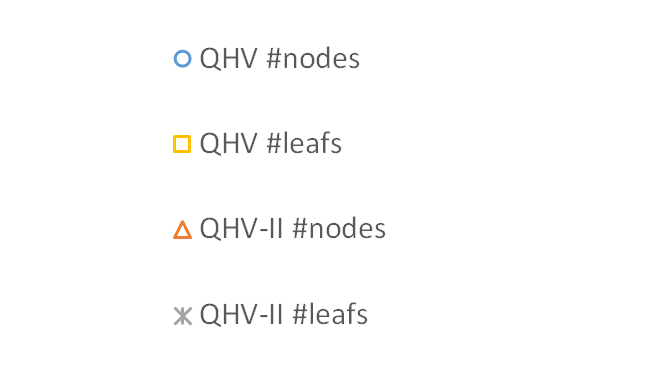}}%

\leavevmode\smash{\makebox[0pt]{\hspace{0em}
  \rotatebox[origin=l]{90}{\hspace{4em}
    Number of operations}%
}}\hspace{0pt plus 1filll}\null

Number of points

\medskip

\caption{Number of nodes and leafs in QHV-cut and QHV-II-cut for the linear instances}
\label{fig:opsL}
\end{figure}

\begin{figure}[htb]
\centering
\hspace*{1.5em}\raisebox{\dimexpr-.5\height-1em}
  {$d=6$}%
\hspace*{13.5em}\raisebox{\dimexpr-.5\height-1em}
  {$d=8$}%
\\[\medskipamount]
\hspace*{0.5em}\raisebox{\dimexpr-.5\height-1em}
  {\includegraphics[scale=0.58]{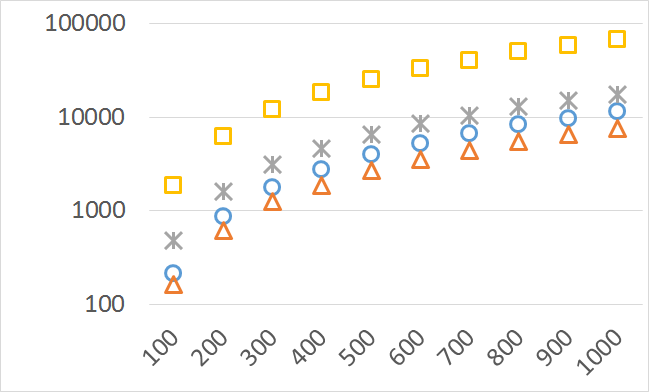}}%
\hspace*{0.2em}\raisebox{\dimexpr-.5\height-1em}
  {\includegraphics[scale=0.58]{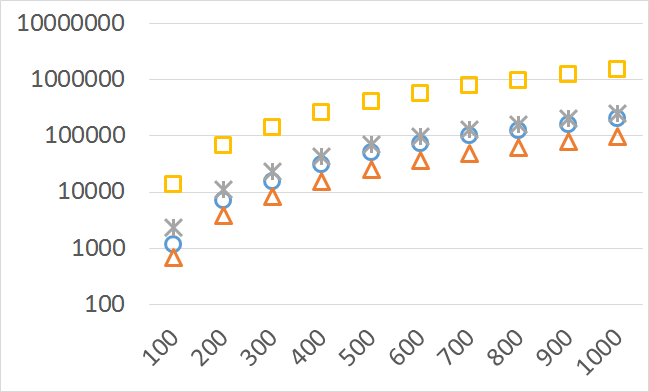}}%
\\[\medskipamount]
\hspace*{1.5em}\raisebox{\dimexpr-.5\height-1em}
  {$d=10$}%
\hspace*{15.5em}\raisebox{\dimexpr-.5\height-1em}
  {$   $}%
\\[\medskipamount]
\hspace*{0.5em}\raisebox{\dimexpr-.5\height-1em}
  {\includegraphics[scale=0.58]{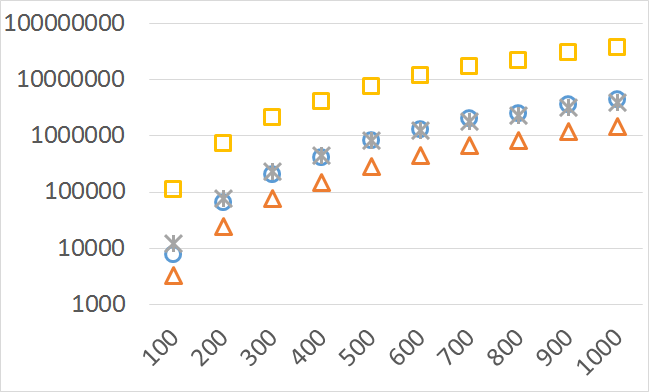}}%
\hspace*{0.2em}\raisebox{\dimexpr-.5\height-1em}
  {\includegraphics[scale=0.58]{opsLegend.png}}%

\leavevmode\smash{\makebox[0pt]{\hspace{0em}
  \rotatebox[origin=l]{90}{\hspace{4em}
    Number of operations}%
}}\hspace{0pt plus 1filll}\null

Number of points

\medskip

\caption{Number of nodes and leafs in QHV-cut and QHV-II-cut for the convex instances}
\label{fig:opsCX}
\end{figure}

\begin{figure}[htb]
\centering
\hspace*{1.5em}\raisebox{\dimexpr-.5\height-1em}
  {$d=6$}%
\hspace*{13.5em}\raisebox{\dimexpr-.5\height-1em}
  {$d=8$}%
\\[\medskipamount]
\hspace*{0.5em}\raisebox{\dimexpr-.5\height-1em}
  {\includegraphics[scale=0.58]{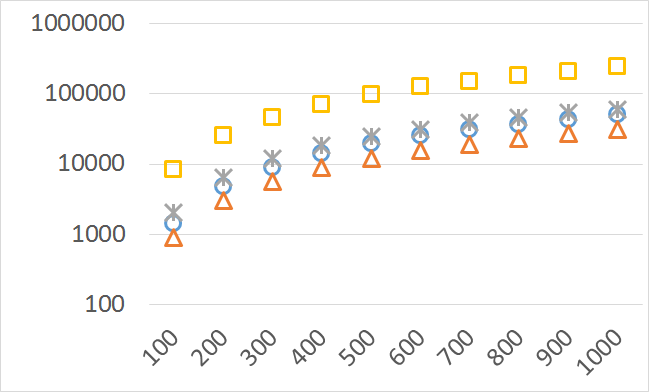}}%
\hspace*{0.2em}\raisebox{\dimexpr-.5\height-1em}
  {\includegraphics[scale=0.58]{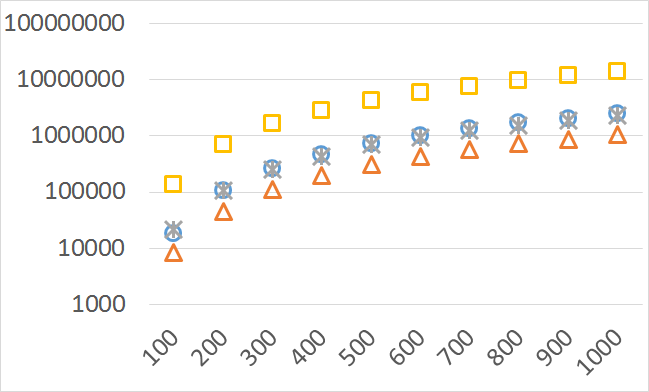}}%
\\[\medskipamount]
\hspace*{1.5em}\raisebox{\dimexpr-.5\height-1em}
  {$d=10$}%
\hspace*{15.5em}\raisebox{\dimexpr-.5\height-1em}
  {$   $}%
\\[\medskipamount]
\hspace*{0.5em}\raisebox{\dimexpr-.5\height-1em}
  {\includegraphics[scale=0.58]{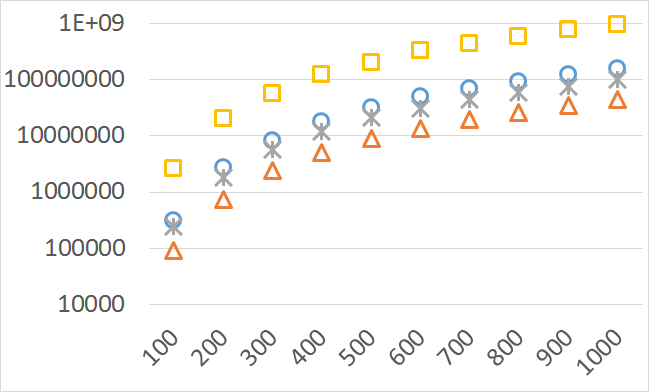}}%
\hspace*{0.2em}\raisebox{\dimexpr-.5\height-1em}
  {\includegraphics[scale=0.58]{opsLegend.png}}%

\leavevmode\smash{\makebox[0pt]{\hspace{0em}
  \rotatebox[origin=l]{90}{\hspace{4em}
    Number of operations}%
}}\hspace{0pt plus 1filll}\null

Number of points

\medskip

\caption{Number of nodes and leafs in QHV-cut and QHV-II-cut for the concave instances}
\label{fig:opsC}
\end{figure}

\begin{figure}[htb]
\centering
\hspace*{1.5em}\raisebox{\dimexpr-.5\height-1em}
  {$d=6$}%
\hspace*{13.5em}\raisebox{\dimexpr-.5\height-1em}
  {$d=8$}%
\\[\medskipamount]
\hspace*{0.5em}\raisebox{\dimexpr-.5\height-1em}
  {\includegraphics[scale=0.58]{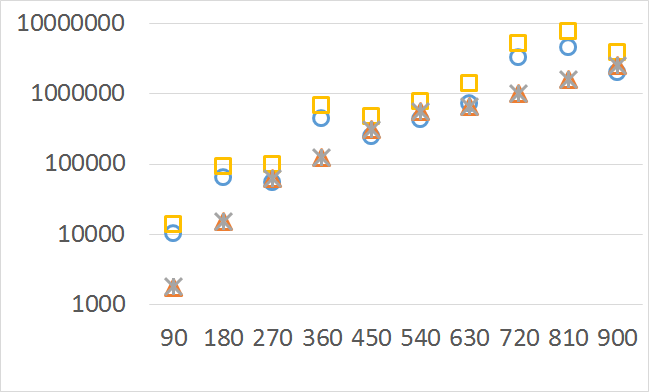}}%
\hspace*{0.2em}\raisebox{\dimexpr-.5\height-1em}
  {\includegraphics[scale=0.58]{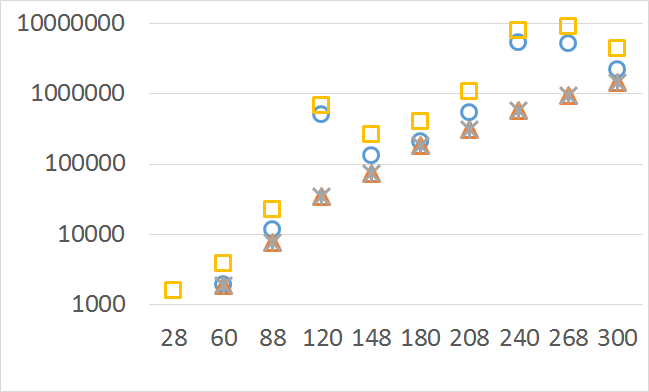}}%
\\[\medskipamount]
\hspace*{1.5em}\raisebox{\dimexpr-.5\height-1em}
  {$d=10$}%
\hspace*{15.5em}\raisebox{\dimexpr-.5\height-1em}
  {$   $}%
\\[\medskipamount]
\hspace*{0.5em}\raisebox{\dimexpr-.5\height-1em}
  {\includegraphics[scale=0.58]{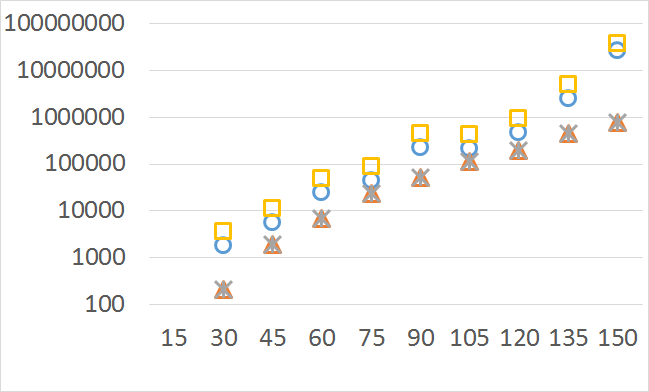}}%
\hspace*{0.2em}\raisebox{\dimexpr-.5\height-1em}
  {\includegraphics[scale=0.58]{opsLegend.png}}%

\leavevmode\smash{\makebox[0pt]{\hspace{0em}
  \rotatebox[origin=l]{90}{\hspace{4em}
    Number of operations}%
}}\hspace{0pt plus 1filll}\null

Number of points

\medskip

\caption{Number of nodes and leafs in QHV-cut and QHV-II-cut for the hard instances}
\label{fig:opsH}
\end{figure}

\begin{figure}[htb]
\centering
\hspace*{1.5em}\raisebox{\dimexpr-.5\height-1em}
  {$d=6$}%
\hspace*{13.5em}\raisebox{\dimexpr-.5\height-1em}
  {$d=8$}%
\\[\medskipamount]
\hspace*{0.5em}\raisebox{\dimexpr-.5\height-1em}
  {\includegraphics[scale=0.58]{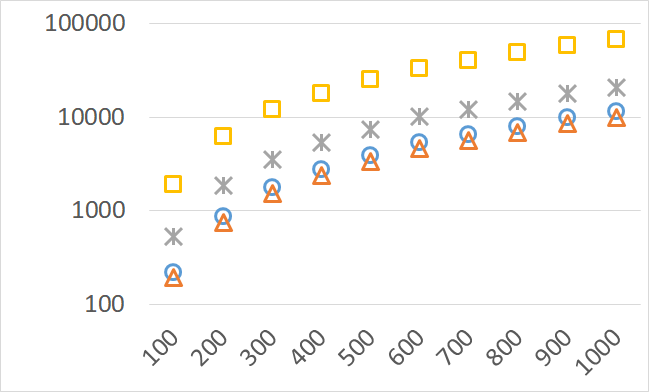}}%
\hspace*{0.2em}\raisebox{\dimexpr-.5\height-1em}
  {\includegraphics[scale=0.58]{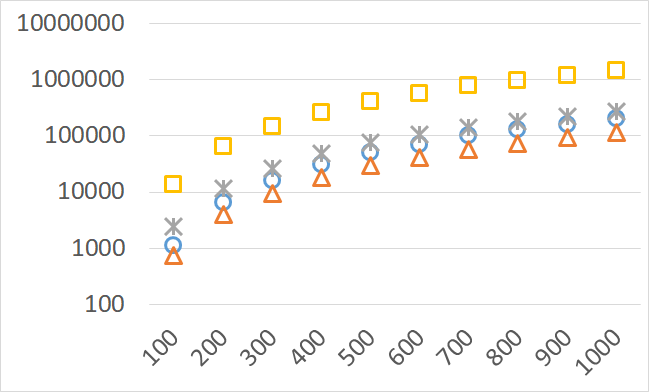}}%
\\[\medskipamount]
\hspace*{1.5em}\raisebox{\dimexpr-.5\height-1em}
  {$d=10$}%
\hspace*{15.5em}\raisebox{\dimexpr-.5\height-1em}
  {$   $}%
\\[\medskipamount]
\hspace*{0.5em}\raisebox{\dimexpr-.5\height-1em}
  {\includegraphics[scale=0.58]{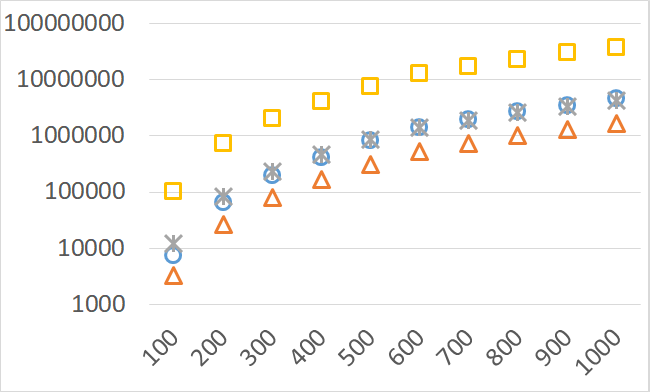}}%
\hspace*{0.2em}\raisebox{\dimexpr-.5\height-1em}
  {\includegraphics[scale=0.58]{opsLegend.png}}%

\leavevmode\smash{\makebox[0pt]{\hspace{0em}
  \rotatebox[origin=l]{90}{\hspace{4em}
    Number of operations}%
}}\hspace{0pt plus 1filll}\null

Number of points

\medskip

\caption{Number of nodes and leafs in QHV-cut and QHV-II-cut for the uniform spherical instances}
\label{fig:opsUS}
\end{figure}

\begin{figure}[htb]
\centering
\hspace*{1.5em}\raisebox{\dimexpr-.5\height-1em}
  {$d=6$}%
\hspace*{13.5em}\raisebox{\dimexpr-.5\height-1em}
  {$d=8$}%
\\[\medskipamount]
\hspace*{0.5em}\raisebox{\dimexpr-.5\height-1em}
  {\includegraphics[scale=0.58]{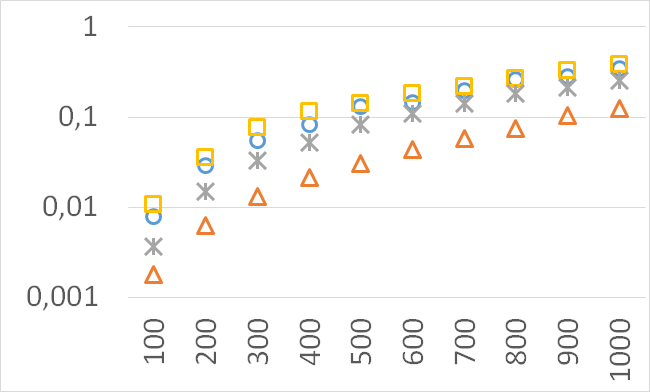}}%
\hspace*{0.2em}\raisebox{\dimexpr-.5\height-1em}
  {\includegraphics[scale=0.58]{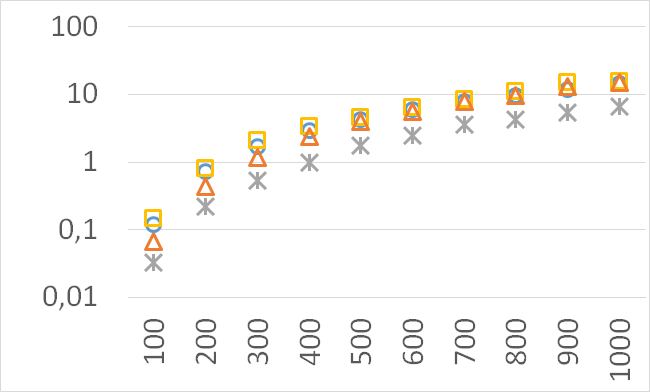}}%
\\[\medskipamount]
\hspace*{1.5em}\raisebox{\dimexpr-.5\height-1em}
  {$d=10$}%
\hspace*{15.5em}\raisebox{\dimexpr-.5\height-1em}
  {$   $}%
\\[\medskipamount]
\hspace*{0.5em}\raisebox{\dimexpr-.5\height-1em}
  {\includegraphics[scale=0.58]{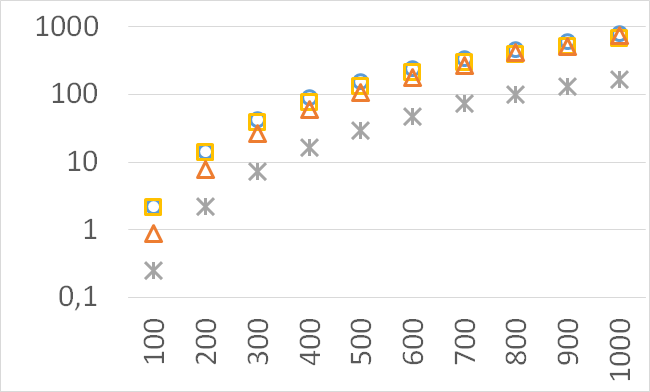}}%
\hspace*{0.2em}\raisebox{\dimexpr-.5\height-1em}
  {\includegraphics[scale=0.58]{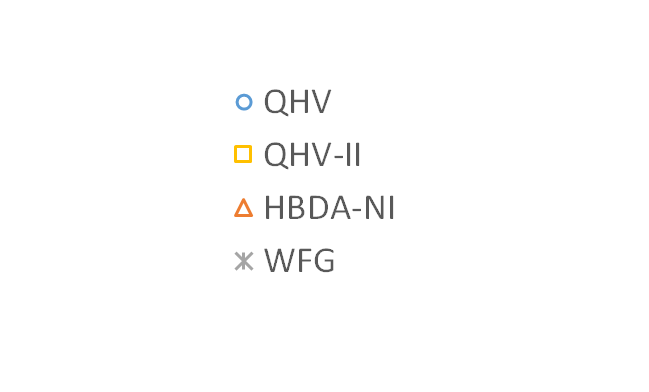}}%

\leavevmode\smash{\makebox[0pt]{\hspace{0em}
  \rotatebox[origin=l]{90}{\hspace{8em}
    CPU time [s]}%
}}\hspace{0pt plus 1filll}\null

Number of points
\medskip
\caption{Running times for the linear instances}
\label{fig:CPUL}
\end{figure}

\begin{figure}[htb]
\centering
\hspace*{1.5em}\raisebox{\dimexpr-.5\height-1em}
  {$d=6$}%
\hspace*{13.5em}\raisebox{\dimexpr-.5\height-1em}
  {$d=8$}%
\\[\medskipamount]
\hspace*{0.5em}\raisebox{\dimexpr-.5\height-1em}
  {\includegraphics[scale=0.58]{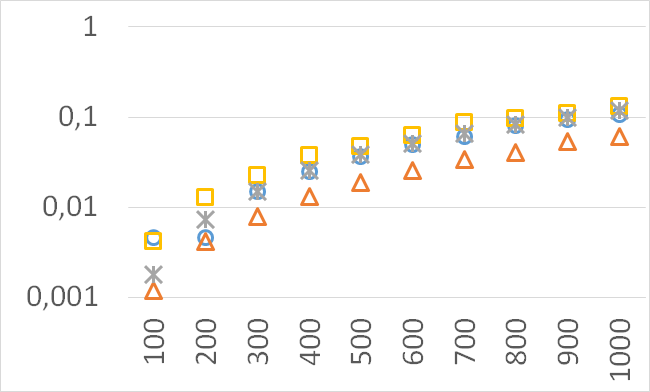}}%
\hspace*{0.2em}\raisebox{\dimexpr-.5\height-1em}
  {\includegraphics[scale=0.58]{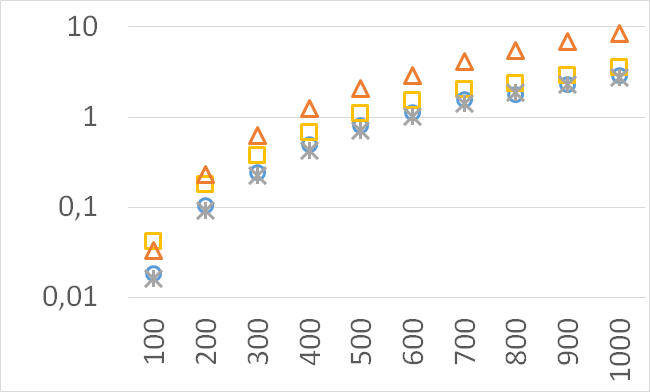}}%
\\[\medskipamount]
\hspace*{1.5em}\raisebox{\dimexpr-.5\height-1em}
  {$d=10$}%
\hspace*{15.5em}\raisebox{\dimexpr-.5\height-1em}
  {$   $}%
\\[\medskipamount]
\hspace*{0.5em}\raisebox{\dimexpr-.5\height-1em}
  {\includegraphics[scale=0.58]{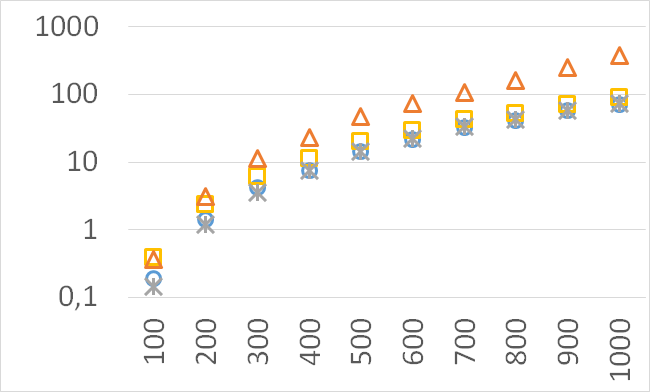}}%
\hspace*{0.2em}\raisebox{\dimexpr-.5\height-1em}
  {\includegraphics[scale=0.58]{CPULegend.png}}%

\leavevmode\smash{\makebox[0pt]{\hspace{0em}
  \rotatebox[origin=l]{90}{\hspace{8em}
    CPU time [s]}%
}}\hspace{0pt plus 1filll}\null

Number of points
\medskip
\caption{Running times for the convex instances}
\label{fig:CPUCX}
\end{figure}

\begin{figure}[htb]
\centering
\hspace*{1.5em}\raisebox{\dimexpr-.5\height-1em}
  {$d=6$}%
\hspace*{13.5em}\raisebox{\dimexpr-.5\height-1em}
  {$d=8$}%
\\[\medskipamount]
\hspace*{0.5em}\raisebox{\dimexpr-.5\height-1em}
  {\includegraphics[scale=0.58]{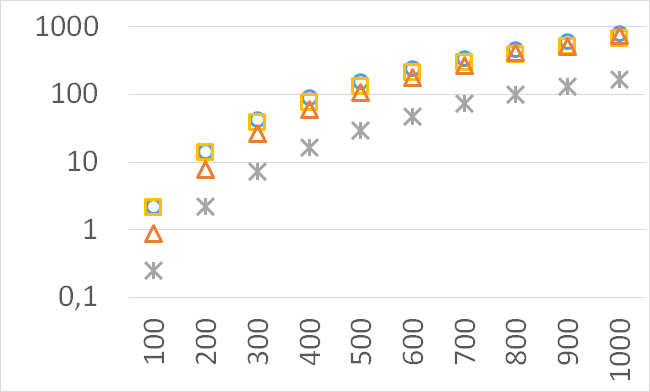}}%
\hspace*{0.2em}\raisebox{\dimexpr-.5\height-1em}
  {\includegraphics[scale=0.58]{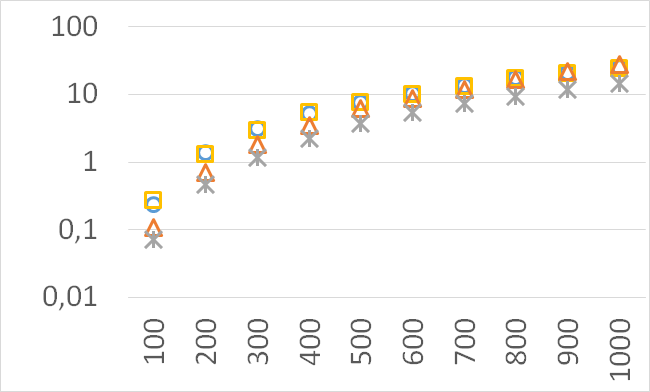}}%
\\[\medskipamount]
\hspace*{1.5em}\raisebox{\dimexpr-.5\height-1em}
  {$d=10$}%
\hspace*{15.5em}\raisebox{\dimexpr-.5\height-1em}
  {$   $}%
\\[\medskipamount]
\hspace*{0.5em}\raisebox{\dimexpr-.5\height-1em}
  {\includegraphics[scale=0.58]{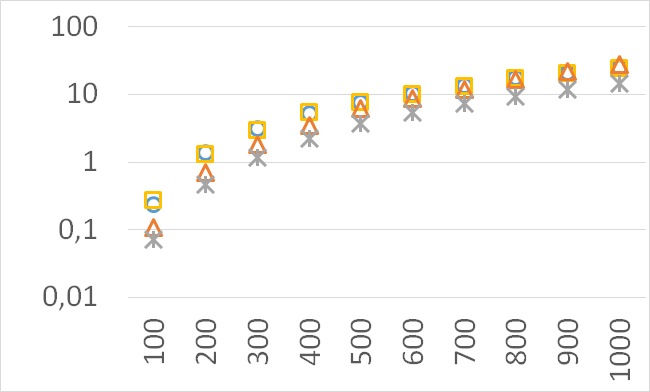}}%
\hspace*{0.2em}\raisebox{\dimexpr-.5\height-1em}
  {\includegraphics[scale=0.58]{CPULegend.png}}%

\leavevmode\smash{\makebox[0pt]{\hspace{0em}
  \rotatebox[origin=l]{90}{\hspace{8em}
    CPU time [s]}%
}}\hspace{0pt plus 1filll}\null

Number of points
\medskip
\caption{Running times for the concave instances}
\label{fig:CPUC}
\end{figure}

\begin{figure}[htb]
\centering
\hspace*{1.5em}\raisebox{\dimexpr-.5\height-1em}
  {$d=6$}%
\hspace*{13.5em}\raisebox{\dimexpr-.5\height-1em}
  {$d=8$}%
\\[\medskipamount]
\hspace*{0.5em}\raisebox{\dimexpr-.5\height-1em}
  {\includegraphics[scale=0.58]{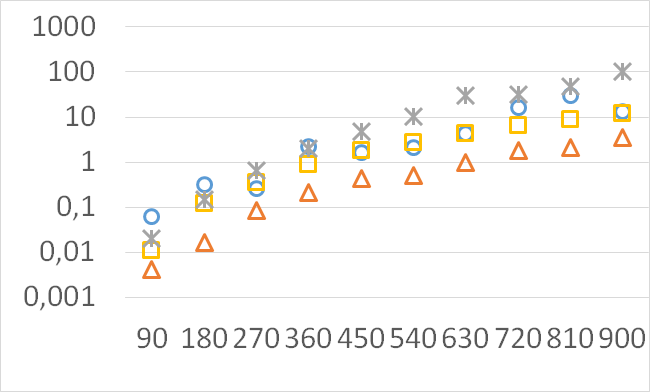}}%
\hspace*{0.2em}\raisebox{\dimexpr-.5\height-1em}
  {\includegraphics[scale=0.58]{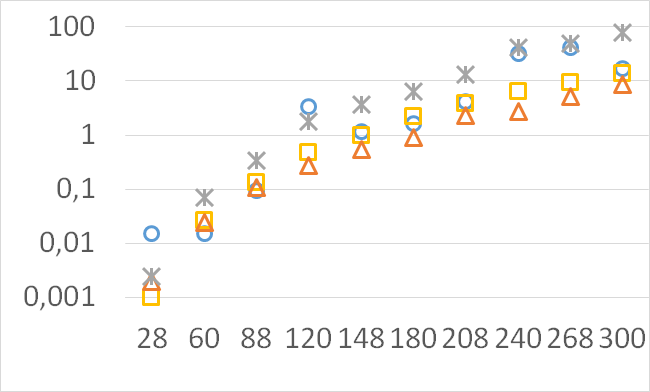}}%
\\[\medskipamount]
\hspace*{1.5em}\raisebox{\dimexpr-.5\height-1em}
  {$d=10$}%
\hspace*{15.5em}\raisebox{\dimexpr-.5\height-1em}
  {$   $}%
\\[\medskipamount]
\hspace*{0.5em}\raisebox{\dimexpr-.5\height-1em}
  {\includegraphics[scale=0.58]{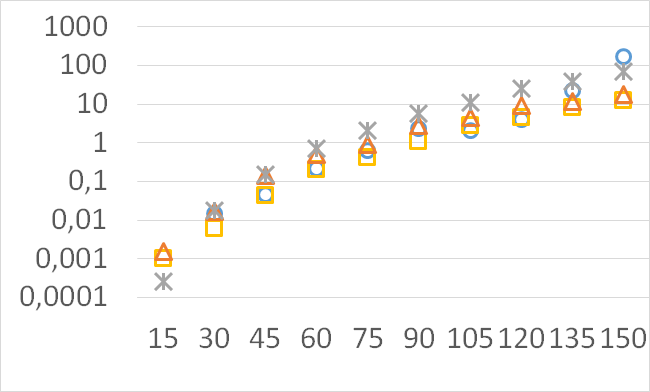}}%
\hspace*{0.2em}\raisebox{\dimexpr-.5\height-1em}
  {\includegraphics[scale=0.58]{CPULegend.png}}%

\leavevmode\smash{\makebox[0pt]{\hspace{0em}
  \rotatebox[origin=l]{90}{\hspace{8em}
    CPU time [s]}%
}}\hspace{0pt plus 1filll}\null

Number of points
\medskip
\caption{Running times for the hard instances}
\label{fig:CPUH}
\end{figure}

\begin{figure}[htb]
\centering
\hspace*{1.5em}\raisebox{\dimexpr-.5\height-1em}
  {$d=6$}%
\hspace*{13.5em}\raisebox{\dimexpr-.5\height-1em}
  {$d=8$}%
\\[\medskipamount]
\hspace*{0.5em}\raisebox{\dimexpr-.5\height-1em}
  {\includegraphics[scale=0.58]{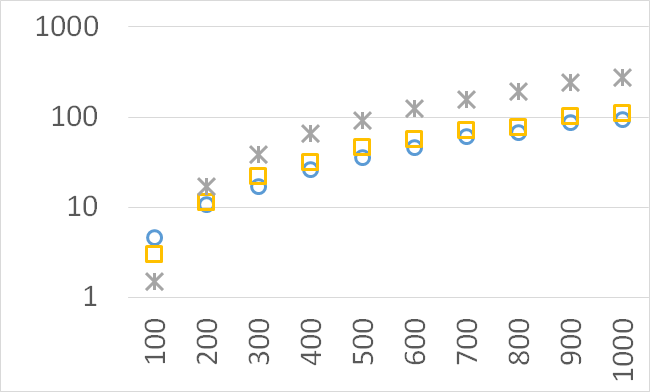}}%
\hspace*{0.2em}\raisebox{\dimexpr-.5\height-1em}
  {\includegraphics[scale=0.58]{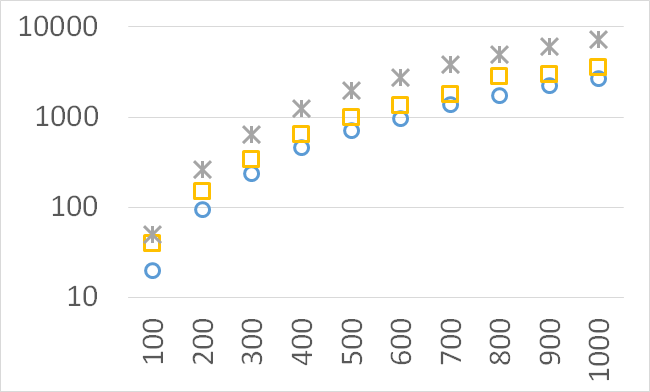}}%
\\[\medskipamount]
\hspace*{1.5em}\raisebox{\dimexpr-.5\height-1em}
  {$d=10$}%
\hspace*{15.5em}\raisebox{\dimexpr-.5\height-1em}
  {$   $}%
\\[\medskipamount]
\hspace*{0.5em}\raisebox{\dimexpr-.5\height-1em}
  {\includegraphics[scale=0.58]{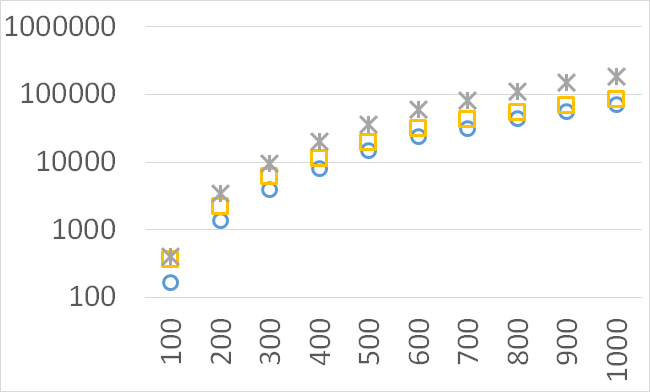}}%
\hspace*{0.2em}\raisebox{\dimexpr-.5\height-1em}
  {\includegraphics[scale=0.58]{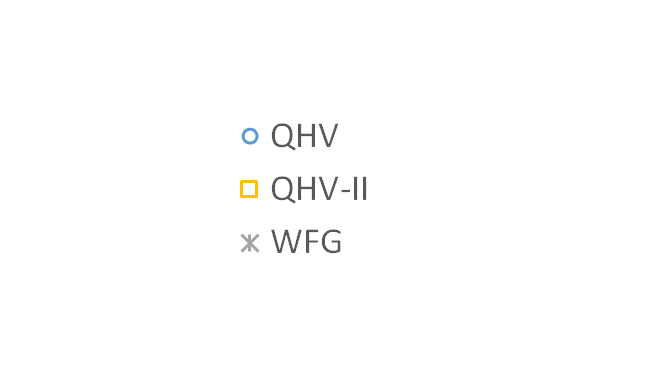}}%

\leavevmode\smash{\makebox[0pt]{\hspace{0em}
  \rotatebox[origin=l]{90}{\hspace{8em}
    CPU time [s]}%
}}\hspace{0pt plus 1filll}\null

Number of points
\medskip
\caption{Running times for the uniform spherical instances}
\label{fig:CPUUS}
\end{figure}

\section{Conclusions and Further Work}
We have presented a modified version of QHV divide and conquer algorithm for calculating the exact hypervolume. Namely, we have modified the scheme of splitting the original problem into smaller sub-problems. Through both theoretical analysis and computational experiments we have shown that the modified version constructs smaller recursive trees and reduces the CPU time.

We have not developed a very efficient implementation of QHV-II but we believe that the we have provided a sufficient evidence that an efficient implementation of QHV-II (i.e. implementation comparable to the efficient implementation of QHV - oQHV) would run several times faster than oQHV for the considered instances.

Interesting directions for further research are:
\begin{itemize}
\item The use of the concepts of QHV-II for the incremental calculation of the hypervolume, i.e. the calculation of the change of the hypervolume after adding new point(s). Similar adaptation of the original QHV algorithm has been proposed in \cite{Russo2016}. We can reasonably expect that the splitting scheme resulting in a faster static algorithm would also reduce the time of the incremental calculation of the hypervolume.  
\item Parallelization of QHV-II which may further improve practical running times. This issue was also considered in the case of the original QHV in \cite{Russo2016}.
\item Adaptation of QHV-II for the approximate calculation of the hypervolume with a guaranteed maximum approximation error.
\end{itemize}

\section*{Acknowledgment}

The research of Andrzej Jaszkiewicz was funded by the the Polish National Science Center, grant no.~UMO-2013/11/B/ST6/01075.

We would like to thank Renaud Lacour, Kathrin Klamroth, and Carlos M. Fonseca for providing us detailed results of the computational experiment reported in \cite{Lacour2017347}.





\section*{References}
\bibliographystyle{model1-num-names}
\bibliography{sample.bib}







\end{document}